\documentclass[letterpaper, 10 pt, conference]{ieeeconf}  

\IEEEoverridecommandlockouts                              

\newcommand{\roa}{\texttt{RoA}}
\newtheorem{thm}{Theorem}[section]

\newcommand{\cB}{{\mathcal B}}

\newcommand{\cD}{\mathcal{D}}

\newcommand{\cF}{{\mathcal F}}

\newcommand{\cZ}{{\mathcal Z}}

\newcommand{\sR}{{\mathsf R}}

\newcommand{\sCG}{\mathsf{CG}}

    \newfont{\mbf}{msbm10 scaled 1100}

\newcommand{\sMG}{{\mathsf{ MG}}}

\usepackage[noadjust,compress]{cite}
\usepackage{graphics} 
\usepackage{graphicx} 
\usepackage{amsmath} 
\usepackage{amssymb}  
\usepackage{amsfonts}
\usepackage{color}
\usepackage{lipsum}
\usepackage{pgfplots}
\usepackage{subcaption}
\pgfplotsset{compat=1.11} 
\usepackage{xcolor}
\usepackage{overpic}
\usepackage{multicol}
\usepackage{multirow}
\usepackage{soul}
\usepackage{wrapfig}
\usepackage{url}
\usepackage{hyperref}
\usepackage{comment}

\definecolor{colorA}{HTML}{B499BA}
\definecolor{colorB}{HTML}{440354}
\definecolor{colorC}{HTML}{FDE725}
\definecolor{colorD}{HTML}{22918C}
\definecolor{colorE}{HTML}{A6D3D1}
\definecolor{colorF}{HTML}{FFFFAD}

\newcommand{\name}{{\tt MORALS}}

\newenvironment{myitem}
{
    \begin{list}{$\circ$ }{}
        \setlength{\topsep}{0pt}
        \setlength{\parskip}{0pt}
        \setlength{\partopsep}{0pt}
        \setlength{\parsep}{0pt}         
        \setlength{\itemsep}{0pt} 
	\setlength{\leftskip}{-12pt}
}
{
    \end{list} 
}

\title{\LARGE \bf
\name: Analysis of High-Dimensional Robot Controllers\\ via Topological Tools in a Latent Space
}

\author{Ewerton R. Vieira$^{*,2,3}$, Aravind Sivaramakrishnan$^{*,1}$, Sumanth Tangirala$^{1}$, Edgar Granados$^1$, \\ Konstantin Mischaikow$^{2}$, and Kostas E. Bekris$^{1}$
\thanks{$^{*}$The first two authors contributed equally to this paper.
}
\thanks{$^{**}$Code: \url{https://github.com/Ewerton-Vieira/MORALS}}
\thanks{$^{1}$ Dept. of Computer Science, $^{2}$ Dept. of Mathematics and $^{3}$ DIMACS, Rutgers, NJ, USA.  {\tt \{er691,as2578, kb572\}@rutgers.edu}.}%
\thanks{This work is partly supported by NSF HDR TRIPODS award 1934924. EV and KM are partially supported by Air Force Office of Scientific Research under award numbers FA9550-23-1-0011 and FA9550-23-1-0400.}
}

\begin{document}

\maketitle
\thispagestyle{empty}
\pagestyle{empty}

\begin{abstract}

Estimating the region of attraction (\roa) for a robot controller is essential for safe application and controller composition. Many existing methods require a closed-form expression that limit applicability to data-driven controllers. Methods that operate only over trajectory rollouts tend to be data-hungry. In prior work, we have demonstrated that topological tools based on  {\it Morse Graphs} (directed acyclic graphs that combinatorially represent the underlying nonlinear dynamics) offer data-efficient \roa\ estimation without needing an analytical model. They struggle, however, with high-dimensional systems as they operate over a state-space discretization. This paper presents {\it Mo}rse Graph-aided discovery of {\it R}egions of {\it A}ttraction in a learned {\it L}atent {\it S}pace (\name)$^{**}$. The approach combines auto-encoding neural networks with Morse Graphs. \name\ shows promising predictive capabilities in estimating attractors and their \roa s for data-driven controllers operating over high-dimensional systems, including a 67-dim humanoid robot and a 96-dim 3-fingered manipulator. It first projects the dynamics of the controlled system into a learned latent space. Then, it constructs a reduced form of Morse Graphs representing the bistability of the underlying dynamics, i.e., detecting when the controller results in a desired versus an undesired behavior. The evaluation on high-dimensional robotic datasets indicates data efficiency in \roa\ estimation.
\end{abstract}


\section{Introduction}
\label{sec:introduction}

Given a controller for a robotic system, it is desirable to estimate its region of attraction (\roa), i.e., a subset of the system's state space, such that all trajectories starting inside this set converge to an equilibrium \cite{bansal2017hamilton}. \roa\ estimation helps understand the conditions under which the controller can be safely applied to solve a task. It can also be used for controller composition where the final \roa\ is larger than the \roa\ of each component controller \cite{tedrake2010lqr}. 

The authors introduced in prior work topological tools based on \textit{Morse graphs} \cite{morsegraph}. Morse Graphs provide a finite, combinatorial representation of the state space given access to a discrete-time representation of the dynamics. They correspond to directed acyclic graphs that provide a rigorous description of attractors and \roa s at different levels of resolution. They were introduced as data-efficient and more accurate alternatives to estimate the \roa s of general robot controllers, including data-driven ones. For systems with unknown dynamics, they can be combined with surrogate modeling to identify the \roa\ for a goal set and describe the global dynamic behavior of a controller \cite{gp_morsegraph}. 

\begin{figure}
    \centering
    \includegraphics[width=.99\columnwidth]{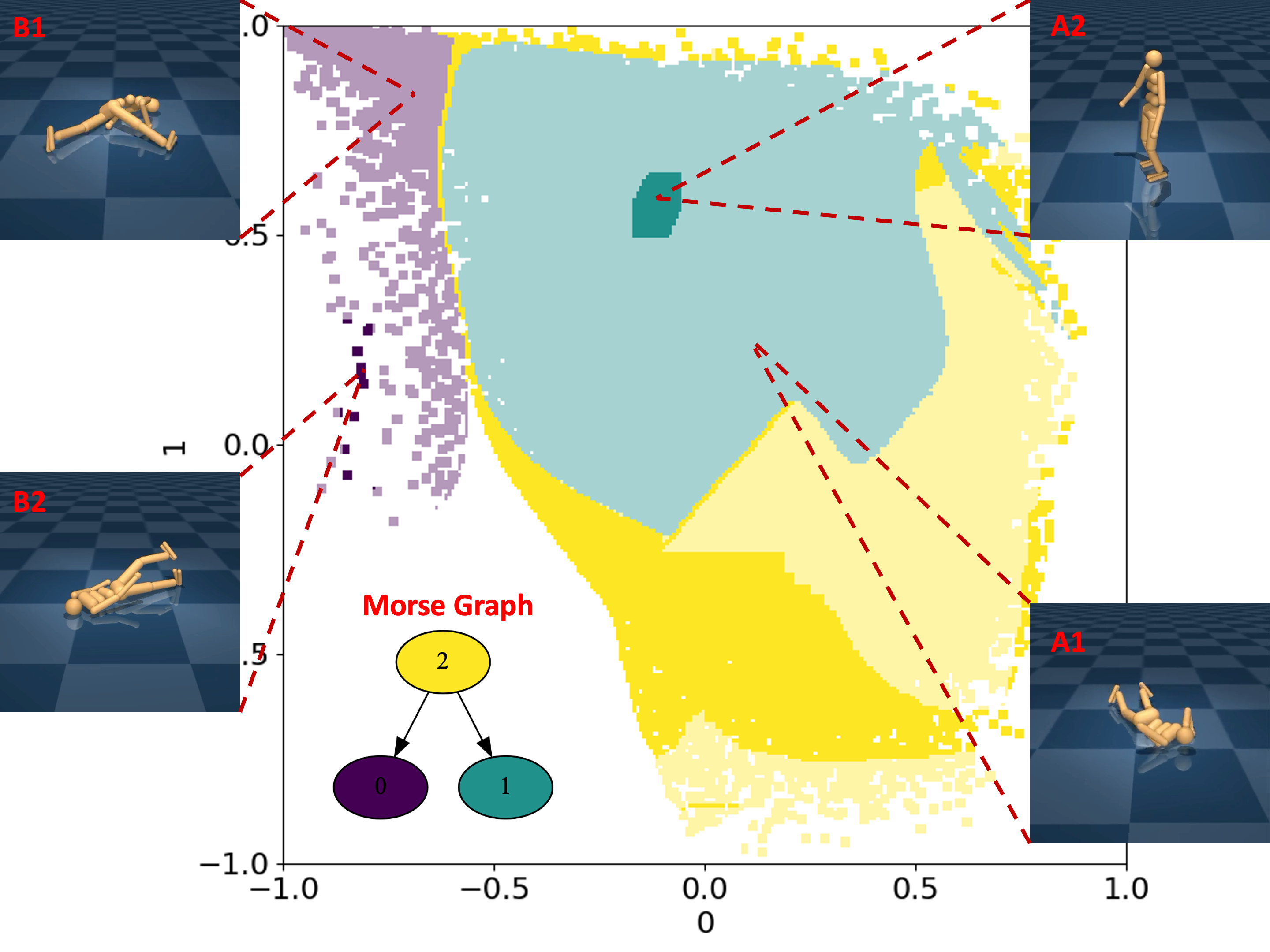}
    \vspace{-.25in}
    \caption{\small The 67-dim. state space of a bipedal humanoid robot controlled by a Soft Actor-Critic (SAC) controller is encoded to a 2-dim. learned latent space by \name, which then discovers two attractors and their corresponding Regions of Attraction (\roa s) in the latent space. Encoded final states {\bf A2}, {\bf B2} are mapped to a desired (dark green) and undesired attractor (dark purple) respectively. Encoded initial states {\bf A1}, {\bf B1} lie respectively in the \roa s (light green and light purple) of the desired and undesired attractors. The yellow region contains the separatrix (undecidable region), indicating initial states may go to {\bf A2} (node 2 $\rightarrow$ node 1) or {\bf B2} (node 2 $\rightarrow$ node 0). Best viewed in color.}
    \label{fig:overview}
\vspace{-0.3in}
\end{figure}

Morse Graphs rely only on point-wise access to short trajectories from each cell of a state space discretization. Their accuracy depends significantly on the size of the discretization of the system's state space. Thus, applying them directly to high-dim. robotic systems, such as bipedal robots (Fig~\ref{fig:overview}), is computationally expensive or even infeasible. In practice, however, the effect of a controller on a robotic system, even a complex and high-dim. one is to restrict the dynamics to a lower-dim. manifold. For example, simplified models of either stiff \cite{kajita20013d} or compliant \cite{geyer2006compliant} inverted pendula can result in controllers for a high-dim. bipedal robot. It may thus be possible to derive meaningful conclusions about the dynamics in the robot's original high-dim. state space by first identifying a lower-dim manifold given example trajectory data and studying the dynamics of that manifold. 

This work proposes \textbf{Mo}rse Graph-aided discovery of \textbf{R}egions of \textbf{A}ttraction in a learned \textbf{L}atent \textbf{Space} (\name), which uses an autoencoding neural network as a lower-dim. surrogate model of the underlying controlled dynamics trained on robot trajectories. The network encodes the system's high-dim. state space to a lower-dim \textit{latent} one so that the latent dynamics locally approximate the system's behavior in the original space. Then, a discretization of the low-dim. space is used to compute a combinatorial representation of the latent dynamics in the form of a Morse Graph, i.e., a directed acyclic graph that provides a rigorous description of the attractors and \roa s at different levels of resolution. Finally, the Morse Graph is reduced into a simpler graph representing the bistability in the dynamics, i.e., identifying desired versus undesired behaviors of the controller as shown in Fig~\ref{fig:overview}.  The experimental evaluation considers a combination of analytical and physically simulated benchmarks, including for a 67-dim. humanoid, as well as real-world robotic datasets - for a 96-dim 3-fingered manipulator. Due to both dimensionality and non-availability of the ground-truth model, it is infeasible to directly apply the prior Morse Graph framework. Nevertheless, \name\ achieves promising accuracy with significantly fewer data requirements in estimating the desired \roa. \emph{To the best of the authors' knowledge, no competing methods currently can perform this analysis for black-box controllers of such high-dim. systems given only trajectory data.}


\section{Related work} 

Estimating the \roa\ of a control system is a hard problem \cite{ahmadi2013complexity}. Multiple methods \cite{giesl2015review} rely on an analytical expression and use linear matrix inequalities \cite{pesterev2017attraction, pesterev2019estimation} or sum-of-squares solvers \cite{parrilo2000structured,tedrake2010lqr,majumdar2017funnel}. Traditional Lyapunov-based methods are applicable \cite{vannelli1985maximal,parrilo2000structured} but require an analytical model. Even data-driven variants tend to require access to point-wise evaluation of the dynamical model \cite{bobiti2016sampling}. Conservative approximation methods try to estimate the largest possible set within the true \roa\ \cite{chen2021learning_hybrid,chen2021learning_nonlinear,mamakoukas2020learning,richards2018lyapunov}. Gaussian Processes can provide Lyapunov-like functions \cite{lederer2019local}. These methods typically suffer from high data requirements and estimation of a single attractor. 

Reachability analysis \cite{bansal2017hamilton} and control barrier functions \cite{ames2019control} are popular alternatives. Reachability analysis can approximate the \roa\ of dynamical walkers \cite{choi2022computation} and together with learning can maintain system safety over a given horizon \cite{gillulay2011guaranteed}. Using Gaussian Processes, a barrier function can be learned to obtain safe policies \cite{akametalu2014reachability}, or identify areas needing exploration for safe set expansion~\cite{wang2018safe}. Controllers can also be jointly trained alongside neural Lyapunov functions \cite{dawson2023safe}. Prior knowledge of an attractor is a common requirement for these approaches. In contrast, the proposed approach can discover multiple if they exist. Alternative methods also lack explainability, while \name\ provides a graphical description of the dynamics. 


Unsupervised representation learning helps extract a latent representation of a robot's state space and enforces the dynamics in this learned space.  Such efforts \cite{srinivas2018universal,banijamali2018robust,watter2015embed} tend to focus on locally valid dynamics and do not study the global dynamics. Latent Sampling-based Motion Planning (L-SBMP) \cite{ichter2019robot} enforces the latent dynamics via a reachability Grammian. Similarly, Learning To Correspond Dynamical Systems (L2CDS) \cite{pmlr-v120-kim20a} learns correspondences between pairs of dynamical systems through a shared latent dynamical system. \name\ uses an autoencoding architecture similar to L-SBMP and L2CDS. Unlike L-SBMP, it does not learn the latent dynamics for the system but focuses on the controller studied. Unlike L2CDS, \name\ does not require assumptions about the latent dynamical model and learns the latent manifold directly from data.

\begin{figure*}[!htbp]
\centering
\scalebox{0.780}{
\begin{overpic}[width=0.4\linewidth]{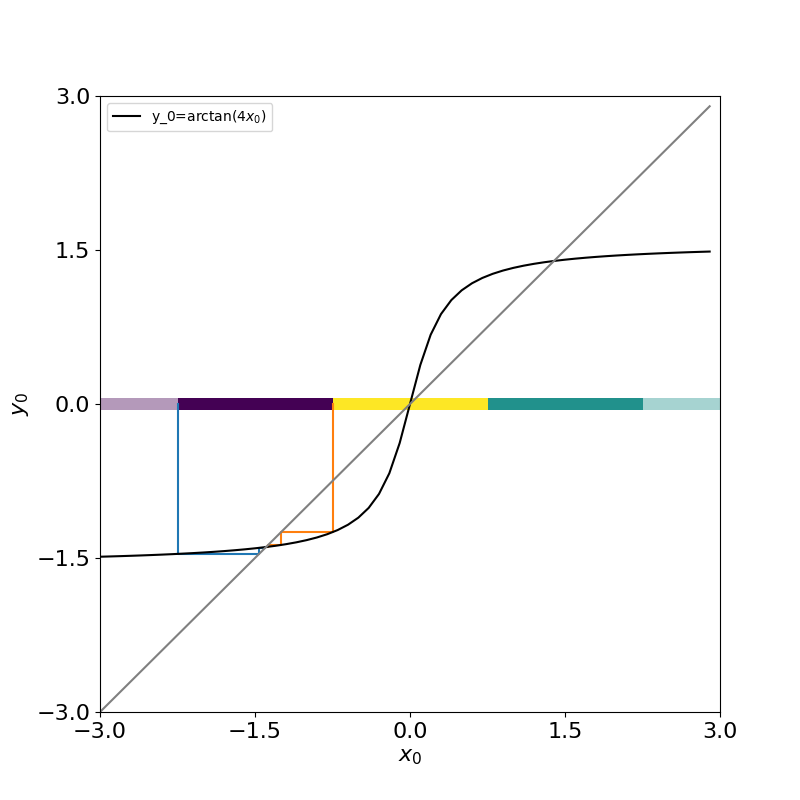}
   \put(15,51){$a$}
   \put(31,51){$b$}
   \put(46,51){$c$}
   \put(70,51){$d$}
   \put(83,51){$e$}
   \put(14,65){
   \scalebox{0.5}{
   \begin{tikzpicture}
   \begin{scope}[every node/.style={circle,thick,draw}]
        \node[fill=colorB] (B) at (1,0) {\textcolor{white}{B}};
        \node[fill=colorC] (C) at (2,1) {C};
        \node[fill=colorD] (D) at (3,0) {\textcolor{white}{D}};
    \end{scope}
    \begin{scope}[every edge/.style={draw=black, very thick}]
        \path [->] (C) edge (B);
        \path [->] (C) edge (D);
    \end{scope}
    \node at (0.7,1) {$\sMG(\cF)$};
    \draw[rounded corners] (0, -0.5) rectangle (3.5, 1.5);
    \end{tikzpicture}
    }
   }
   \put(50,15){
   \scalebox{0.5}{
   \begin{tikzpicture}
    \begin{scope}[every node/.style={circle,thick,draw}]
        \node[fill=colorA] (A) at (0,1) {A};
        \node[fill=colorB] (B) at (1,0) {\textcolor{white}{B}};
        \node[fill=colorC] (C) at (2,1) {C};
        \node[fill=colorD] (D) at (3,0) {\textcolor{white}{D}};
        \node[fill=colorE] (E) at (4,1) {E};
        \node[fill=colorF] (F) at (3,1.2) {F};
    \end{scope}
    \begin{scope}[every edge/.style={draw=black, very thick}]
        \path [->] (A) edge (B);
        \path [->] (C) edge (B);
        \path [->] (C) edge (D);
        \path [->] (E) edge (D);
        \path [->] (F) edge (C);
        \path [->] (C) edge [loop above] (C);
        \path [->] (B) edge [loop below] (B);
        \path [->] (D) edge [loop below] (D);
    \end{scope}
    \draw[rounded corners] (-0.5, -1) rectangle (4.5, 2);
    \node at (0.5,1.8) {$\mathcal{F}$};
    \draw[draw=none, fill=blue!30,rounded corners, opacity=0.2] (-0.3, -0.5) rectangle (4.3, 1.5);
    \node at (0.5,1.8) {$\mathcal{F}$};
    \node at (4,0) {\textcolor{blue}{$\sCG(\cF)$}};
    \end{tikzpicture}
    }
   }
\end{overpic}
\begin{overpic}[width=0.4\linewidth]{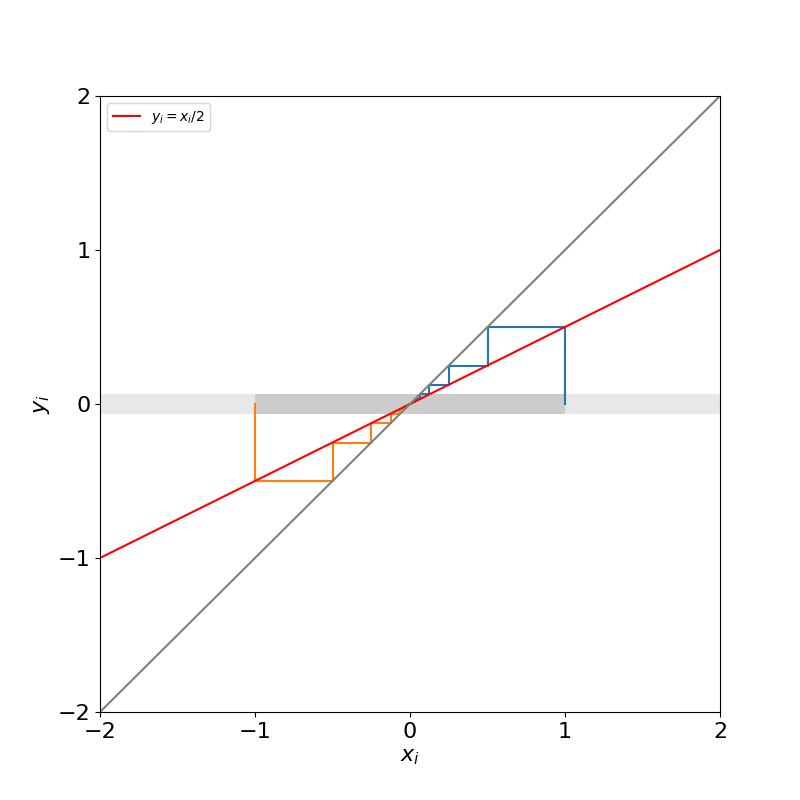}
    \put(-6,47){\scalebox{2.5}{$\times$}}
   \put(20,51){$f_i$}
   \put(42,51){$g_i$}
   \put(80,51){$h_i$}
   \put(-6,83){\scalebox{2.5}{$X$}}
    \put(-109,-0.5){
   \begin{tikzpicture}
      \draw[rounded corners, line width=0.25mm] (0, 0) rectangle (14.8, 6.8);
    \end{tikzpicture}
    }
\end{overpic}
}\quad
\begin{overpic}[width=0.30\linewidth]{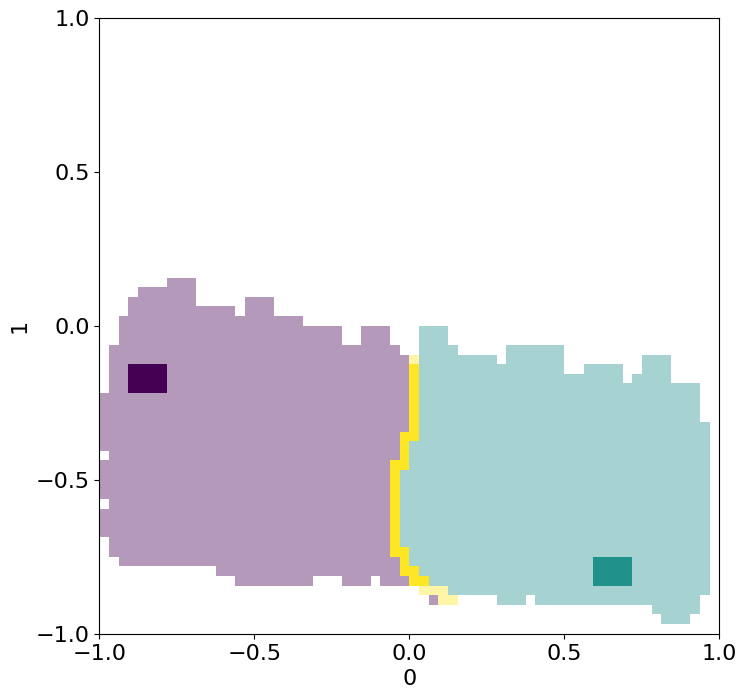}
\put(60,70){
 \scalebox{0.5}{
   \begin{tikzpicture}
   \begin{scope}[every node/.style={circle,thick,draw}]
        \node[fill=colorB] (B) at (1,0) {\textcolor{white}{0}};
        \node[fill=colorC] (C) at (2,1) {2};
        \node[fill=colorD] (D) at (3,0) {\textcolor{white}{1}};
    \end{scope}
    \begin{scope}[every edge/.style={draw=black, very thick}]
        \path [->] (C) edge (B);
        \path [->] (C) edge (D);
    \end{scope}
    \node at (0.7,1) {$\sR_{\sMG}$};
    \draw[rounded corners] (0, -0.5) rectangle (3.5, 1.5);
    \end{tikzpicture}
    }
}
\end{overpic}
\caption{\small Example of $N$-dim. bistability (left and middle) and the learned dynamics on the 2-dim. encoded space (right). (left \& middle) The state space is $X=X_1 \times \prod_{i=2}^N X_i$ where $X_i=[-3,3]\times[-2,2]^{N-1} $ and the dynamics $\phi : X \to X$ is given by $\phi_1(x) = \arctan(4x)$ plotted in black and $\phi_i(x) = x/2$ plotted in red, where $i=2, \ldots, N$. The domains $X_1 = [-3,3]$ and $X_i=[-2,2]$ are decomposed into intervals $a$ through $e$ and $f_i, g_i, h_i$, respectively. Forward propagation of $B=b\times \prod g_i$ is depicted by the lines from
the boundary of $b$ and $g_i$'s. $\cF$ is a directed graph capturing reachable vertices from other vertices, (regions in $\{a,b,c,d,e\}\times\prod_{i=2}^N\{f_i,g_i,h_i\}$). Strongly connected components of $\cF$ result in $\sCG(\cF)$. Finally, the Morse Graph $\sMG(\cF)$ (nodes {B, C, D}) contains the attractors and expresses their \roa s.  (right) The $N$-dim. bistability dynamics are encoded into a 2-dim. latent space represented by a bistable Morse Graph $\sR_{\sMG}$.}
\vspace{-.2in}
\label{fig:bistable}
\end{figure*}
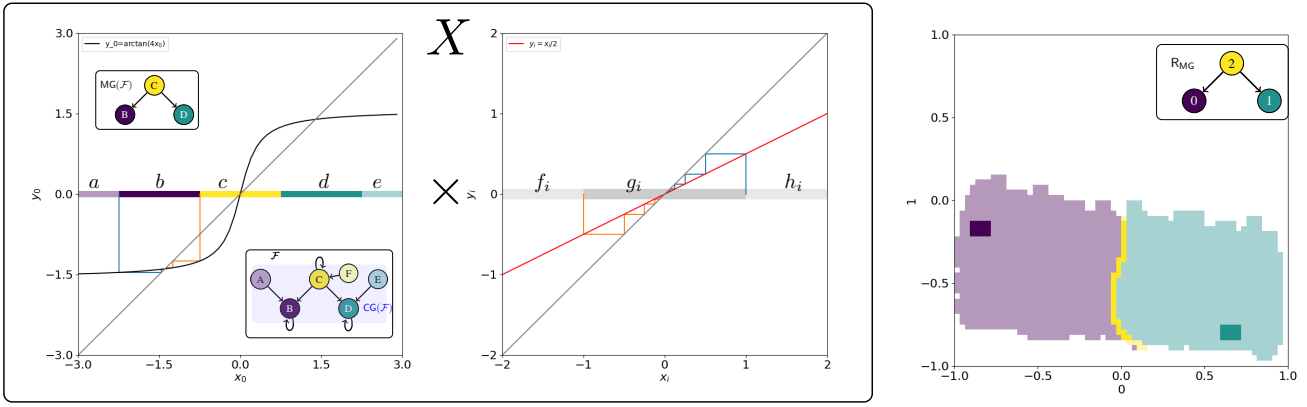 

\section{Preliminaries}
\label{sec:prelims}

This work aims to provide a data-efficient framework for the analysis of the global dynamics of robot controllers based on combinatorial dynamics and order theory \cite{morsegraph, kalies:mischaikow:vandervorst:14,kalies:mischaikow:vandervorst:15,kalies:mischaikow:vandervorst:21}. Consider a non-linear, continuous-time control system: \vspace{-0.15in}

\begin{equation} \label{eq:dynamics}
    \dot{x} = f(x,u)
    \vspace{-0.05in}
\end{equation}

\noindent where $x(t) \in X \subseteq \mathbb{R}^N$ is the state at time $t$, $X$ is a compact set, $u: X \mapsto U \subseteq \mathbb{R}^M$ is a Lipschitz-continuous control as defined by a deterministic control policy $u(x)$, and $f: X \times \mathbb{U} \mapsto \mathbb{R}^M$ is a Lipschitz-continuous function. Neither $f(\cdot)$ nor $u(x)$ are necessarily known analytically. For instance, 
$u(x)$ can be a function learned via a neural network. 

A trajectory (or an \textit{orbit}) of length $\tau > 0$ is defined as a sequence of states obtained by integrating  Eq.~\ref{eq:dynamics} forward in time. Let the \textit{image} $\phi_\tau: X \rightarrow X$ denote the function obtained by mapping every initial state $x_0 \in X$ to the end state of a trajectory of length $\tau$ beginning at $x_0.$ A set $A$ is an attractor if there exists a neighborhood $N$ of $A$ such that \vspace{-.1in}
\[
\mathcal{A}=\omega(N):= \bigcap_{n \in \mathbb{Z}^+} \mathrm{cl}\left(\bigcup_{k=n}^\infty \phi_\tau^k(N)\right)
\vspace{-.1in} \]
where $\phi_\tau^k$ is the composition $\phi_\tau\circ \cdots \circ\phi_\tau$ ($k$ times) and $\mathrm{cl}$ is topological closure. For instance, attractors can be fixed points, such as a desired goal 
that the control policy manages to bring the system to; or limit
cycles, such as a periodic behavior of the system. A {\it Region of Attraction} (\roa) of an attractor $A$ is a neighborhood of $A$ and a subset of $\cB$ the {\it basin of attraction} $A$, where $\cB$ is the largest set of points whose forward orbits converge to $A$, more specifically, the maximal set $\mathcal{B}$ that satisfies $A=\omega(\cB)$. Given that $f$ and $u$ are Lipschitz-continuous, $\phi_\tau$ is too; also any \roa \  of Eq. \eqref{eq:dynamics} is  an \roa \ under $\phi_\tau$. Thus, it is sufficient to analyze the dynamics according to $\phi_\tau$ to study Eq. \eqref{eq:dynamics}, even if it is not accessible and computable.

As a pedagogical example, consider the $N$-dim. \textit{bistable} system in Fig.~\ref{fig:bistable}: Given a state $x = [x_1, \cdots , x_{N}] \in X = X_1 \times \prod_{i=2}^N X_i=[-3,3]\times[-2,2]^{N-1}$, its image is given by $\phi_\tau(x) = \arctan4x_1\times \prod_{i=2}^N x_i/2$. Obtaining Morse graphs and \roa s involves a four-step procedure:

\emph{1. State space decomposition and outer approximation of $\phi_\tau$.} The function $\phi_\tau$ is approximated by decomposing $X$ into a collection of regions $\mathcal{X}$. For instance, via a grid. Fig. \ref{fig:bistable} (left) shows the intervals $[-3,3]$ and $[-2,2]$ decomposed into sub-intervals $a$ to $e$ and $f_i$ to $h_i$ respectively. A grid will be the Cartesian product $\{a,b,c,d,e\}\times\prod_{i=2}^N\{f_i,g_i,h_i\}$. Given a region $\xi \in \mathcal{X}$ (i.e., a cell in the grid), the system is forward propagated for multiple initial states within $\xi$ for time $\tau$ to identify regions reachable from $\xi$. Consider the sub-interval $b\times \prod g_i$. The lines from the boundary of $b\times \prod g_i$ depict the forward propagation of the dynamics. This cell maps to itself given the underlying dynamics.

\emph{2. Constructing combinatorial representation $\mathcal{F}$ of the dynamics.}  The directed graph representation denoted as $\mathcal{F}$ stores regions $\xi \in \mathcal{X}$  as  vertices. Edges from $\xi$ point to regions reachable from $\xi$.  In Fig. \ref{fig:bistable}, the graph $\mathcal{F}$ contains the following nodes/cells: $B=\{b\}\times\prod_{i=2}^N\{g_i\}$, $C=\{c\}\times\prod_{i=2}^N\{g_i\}$, $D=\{d\}\times\prod_{i=2}^N\{g_i\}$, $A=\{a,b\}\times\prod_{i=2}^N\{f_i,g_i,h_i\}\backslash B$, 
$E=\{d,e\}\times\prod_{i=2}^N\{f_i,g_i,h_i\}\backslash D$ and $F=\{c\}\times\prod_{i=2}^N\{f_i,g_i,h_i\}\backslash C$. The edges $(C, C)$ and $(C, B)$ express that $C$ maps both to itself and to $B$. 


\emph{3. Compute the \textit{Condensation Graph} $\mathsf{CG}(\mathcal{F})$}. Collapsing all nodes that are part of strongly connected components (SCCs) in $\mathcal{F}$ into a single node gives rise to the condensation graph $\mathrm{CG}\mathcal{(F)}$. Edges on $\mathrm{CG}\mathcal{(F)}$ reflect reachability due to the dynamics given a topological sort on $\mathcal{F}$. In Fig. \ref{fig:bistable}(left), $\mathrm{CG}\mathcal{(F)}$ is the subgraph with nodes $A$ to $F$ and all non-self edges, i.e., $\mathrm{CG}\mathcal{(F)}$ has no cycles. Since $\mathrm{CG}\mathcal{(F)}$ is a directed acyclic graph, it is a {\it partially ordered set}, i.e., a {\it poset}. 
 
{\emph 4. Compute the \textit{Morse Graph} $\mathsf{MG}$ to detect attractors and \roa s.}  A \emph{recurrent set} is a SCC of $\mathsf{CG}(\mathcal{F})$ that contains at least one edge. Then, the \emph{Morse graph} $\mathrm{MG}(\mathcal{F})$ of $\mathcal{F}$ is the subgraph of recurrent sets of $\mathrm{CG}\mathcal{(F)}$. In Fig. \ref{fig:bistable}(left), $\mathrm{MG}(\mathcal{F})$ contains nodes $B$, $C$ and $D$ and the edges between them.  $\mathrm{MG}(\mathcal{F})$ captures both recurrent and non-recurrent dynamics. Recurrent sets are vertices, and the minimal vertices contain attractors of interest. Edges signify reachability between these sets. In Fig. \ref{fig:bistable}, there are 2 attractors: B and D. Cells in A are in the \roa\ of B, and cells in E are in the \roa\ of D. From cells in F and C, the system can end up in either B or D, characterizing a {\it bistability}.

As a summary of prior results with the Morse Graph approach, consider a pendulum governed by  $ml^2\ddot{\theta} = mGl\sin\theta - \beta\dot{\theta} + u$, given mass $m$, gravity $G$, pole length $l$ and friction coefficient $\beta$. It is controlled by a Linear Quadratic Regulator (LQR) to stand upright. Table~\ref{tab:comparisons} reports the \roa\ estimate accuracy (as a ratio over the volume of the true \roa) for the above approach ({\tt MorseGraph}) and alternatives in the literature. Data efficiency is measured using the total propagation steps required for the \roa\ estimate. Both {\tt L-LQR} and {\tt L-SOS} are analytical methods that use a linearized unconstrained form of the dynamics \cite{prajna2002introducing} to obtain a Lyapunov function (LF).  Lyapunov Neural Network ({\tt L-NN}) \cite{richards2018lyapunov} is a machine learning tool for estimating \roa s.

\begin{table}[h!]
\vspace{-0.1in}
    \centering
    \begin{tabular}{|c|c|c|c|c|}
    \hline
        \textbf{Metric} & {\tt L-LQR} & {\tt L-SOS} & {\tt L-NN} & {\tt MorseGraph}  \\ \hline
         Accuracy & 70\% & 3\% & 98\% & 97\% \\ \hline 
         Prop. steps & $-$ & $-$  & 667.1M & 120K \\ \hline 
    \end{tabular}
\vspace{-0.1cm}
    \caption{\small Accuracy and data efficiency of desired \roa\ estimate for Pend (LQR) using different methods. {\tt L-LQR} and {\tt L-SOS} \cite{morsegraph} require the analytical form of the dynamics. }
    \label{tab:comparisons}
\vspace{-0.15in}
\end{table}

The above methods either require access to an analytical model ({\tt L-LQR}, {\tt L-SOS}) or dynamics propagation from a dense set of initial states ({\tt L-NN}). This may not be possible for complex and high-dim. robots or data-driven controllers that do not admit a closed-form expression.  As the dimension of the underlying system increases, it becomes challenging to apply the {\tt MorseGraph} approach, even though it requires $6000\times$ fewer data points than {\tt L-NN} for a comparable \roa \ estimate. This is due to the exponential increase in the number of discretization elements. The proposed framework deals with this issue via unsupervised representation learning.
\begin{figure*}[h!]
    \centering
    \begin{subfigure}{.33\textwidth}
    \includegraphics[width=\textwidth]{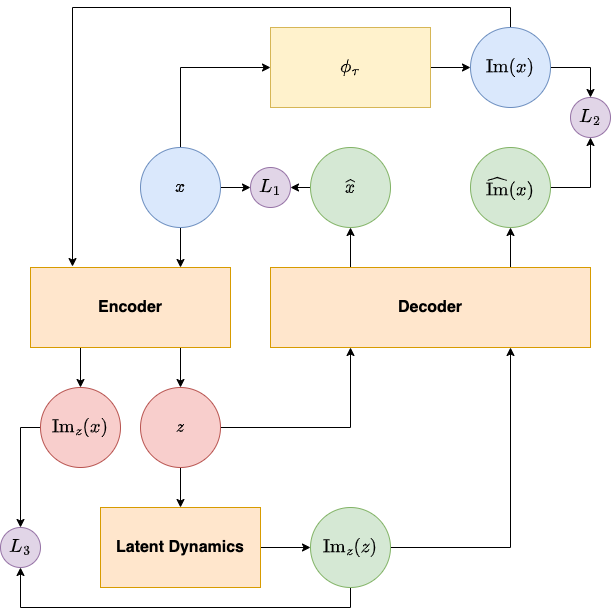}
    \end{subfigure}
    \begin{subfigure}{.66\textwidth}
    \includegraphics[width=\textwidth]{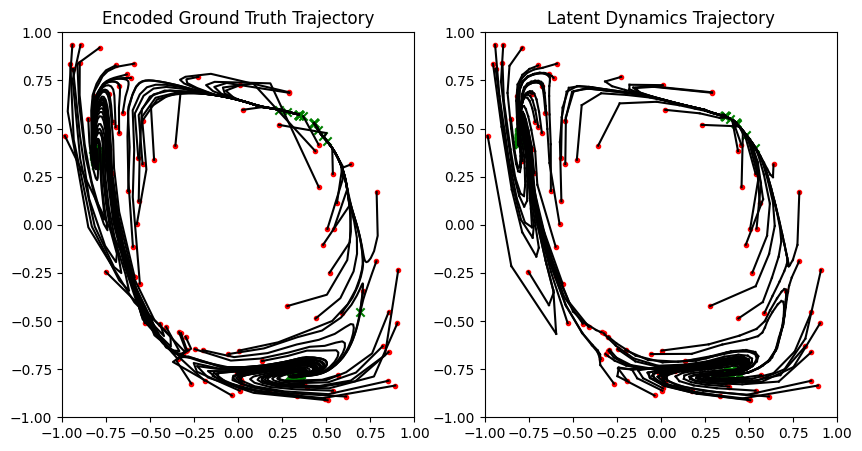}
    \end{subfigure}
    \vspace{-.2in}
    \caption{\small (Left) The autoencoding neural network and loss functions used for training the encoder $h_\text{enc}$, decoder $h_\text{dec},$ and the latent dynamics $h_\text{dyn}$. (Middle/Right) Visualizing the learned latent dynamics for a 4-dim. version of a pendulum $(x,\dot{x},y,\dot{y})$ controlled by LQR. (Middle) the ground-truth trajectory for the same initial conditions (circles). (Right) iteratively calling $h_\text{dyn}$ for a fixed number of timesteps. Both plots capture the 3 true attractors ($\times$). The right plot does not contain regions where the trajectories move from one \roa \ to another.}
    \label{fig:main-figure}
    \vspace{-.25in}
\end{figure*}

\section{Proposed Method}
\label{sec:proposed}


Given a high $N$-dim. state $x$, an \textit{encoder} $h_\text{enc}: X \mapsto Z \subseteq \mathbb{R}^D (D < N)$ encodes $x$ to a lower $D$-dim. latent state $z = h_\text{enc}(x)$. A \textit{decoder} performs the inverse mapping $x = h_\text{dec}(z)$. A \textit{latent dynamics} function, $h_\text{dyn}: Z \mapsto Z$, expresses the dynamics of the latent space.  The proposed approach, \name, incorporates an autoencoding network consisting of $h_\text{enc}$, $h_\text{dec}$, and $h_\text{dyn}$, that can be trained from trajectory rollouts of the underlying system. The architecture used is shown in Fig. \ref{fig:main-figure} (left). The trained networks are used to build a combinatorial representation of the dynamics in the learned latent space $Z$, which can be used to understand the \roa s of the system in the original state space $X$. The method is divided into the following steps:

\begin{myitem}
    \item[1.] Collect data from the system.
    \item[2.] Train the neural networks $h_\text{enc}, h_\text{dec}$ and $h_\text{dyn}$.
    \item[3.] Compute the Morse Graph {\tt MG} of the lower dim. system with state space $Z$, whose initial condition is $z_0 = h_\text{enc}(x_0)$ and $\phi_\tau^Z = h_\text{dyn} \circ \cdots \circ h_\text{dyn}$ ($\tau$ times).
    \item[4.] Given a new state $x \in X$, determine whether it is in the desired \roa \ using {\tt MG}.
\end{myitem}

{\bf 1. Data Collection} The method collects time series data in the form of long trajectories from various initial states in $X$. The data is partitioned into ordered pairs of the form $(x, \text{Im}(x))$, where Im($x$) = $\phi_\tau(x)$. Denote by $\mathcal{D} = \{ (x^i, \text{Im}^i(x)) \}_{i=1}^n$ the robot trajectory data collected. 

Let $\mathbf{F}$ be the set of all final states of every trajectory in the dataset. If the underlying system has the notion of a ``desired'' attractor (such as the upright position for a bipedal humanoid), each final state $x_\mathcal{T} \in \mathbf{F}$ is assigned a label $y(x_\mathcal{T})$ of $+1$ if it successfully achieves the desired task or $0$ otherwise. This helps determine which attractors discovered by \name \ in the learned latent manifold are desirable.

\label{sec:network-arch}
\noindent {\bf 2. Network Architecture and Training} Given a pair $(x, \text{Im}(x))$, the encoder network $h_\text{enc}$ maps it onto the latent space to obtain $(z, \text{Im}_z(x)).$ The decoder network $h_\text{dec}$ acts on $z$ to obtain $\hat{x}$, which is the \textit{reconstruction} of input state $x$. The latent dynamics network acts on $z$ to produce the \textit{image} $\text{Im}_z(z)$ in the latent state. The decoder obtains $\hat{\text{Im}}(x)$, the reconstruction of the input Im($x$) by acting on $\text{Im}_z(z)$.

Training attempts to minimize the following losses.
\vspace{-.15in}

\begin{equation} \label{eqn: l1}
    L_1 = \mathbb{E}_{x \sim \mathcal{D}} \vert \vert x - h_\text{dec}(h_\text{enc}(x)) \vert \vert_2^2
\end{equation}
\vspace{-.2in}
\begin{equation} \label{eqn: l2}
    L_2 = \mathbb{E}_{\text{Im}(x) \sim \mathcal{D}} \vert \vert \text{Im}(x) - h_\text{dec}(h_\text{enc}(\text{Im}(x)))
 \vert \vert_2^2
\end{equation}
 \vspace{-.2in}
\begin{equation} \label{eqn: l3}    
    L_3 = \mathbb{E}_{(x, \text{Im}(x)) \sim \mathcal{D}} \vert \vert h_\text{dyn}(h_\text{enc}(x)) - h_\text{enc}(\text{Im}(x)) \vert \vert_2^2
 \vspace{-.05in}
\end{equation}

$L_1$ and $L_2$ losses enforce that the reconstructed inputs closely match the original ones  by minimizing Euclidean distance. $L_3$ enforces the local dynamics of the latent space by minimizing the Euclidean distance between the latent state image and the encoded image.

All three neural networks are jointly trained to minimize $L = \lambda_1 L_1 + \lambda_2 L_2 + \lambda_3 L_3$, where $\lambda_i, i \in  [1,2,3]$ are weights for each minimization objective. To ensure that the obtained latent manifold is bounded, the output layers of $h_\text{enc}$ and $h_\text{dyn}$ are activated using the $\tanh$ activation. Thus, the obtained latent manifold always lies inside $[-1,1]^D$.

\emph{Optional loss term using labeled data.} Split the final states of demonstrated trajectories into  desirable and undesirable sets: $\mathbf{F}_\text{s} = \{x \in \mathbf{F} \vert y(x) = +1\}$ and $\mathbf{F}_\text{f} = \{x \in \mathbf{F} \vert y(x) = 0\}$. Then, an $L_4$ loss can separates the encoded final states in $\mathbf{F}_s$ from their counterparts in $\mathbf{F}_f$: \vspace{-0.5cm}

\begin{equation} \label{eqn: l4}
    L_4 = \mathbb{E}_{x_s \sim \mathbf{F}_s, x_f \sim \mathbf{F}_f} [\sigma(-c\vert \vert h_\text{enc}(x_s) - h_\text{enc}(x_f) \vert \vert)].
\end{equation}

To ensure that the encoded final states do not lie on the latent space boundaries, the sigmoid function $\sigma(a) = \frac{1}{1 - e^{-a}}$ is used to upper-bound the distance between the encodings of the desirable and undesirable final states. The constant $c$ is a scaling factor. The training procedure iterates between minimizing the loss functions $\lambda_1L_1 + \lambda_2L_2 + \lambda_3L_3$ and $\lambda_4L_4$.

\noindent {\bf 3. Morse Graph computation in the latent space} The uniform discretization $\mathcal{Z}$ of the latent space $[-1,1]^D$ into $\prod_{i=1}^D 2^{k_i}$ cubes of dimension $D$ needs to be validated since some cubes in $\mathcal{Z}$ may not correspond to a valid state in $X$ when decoded. To ensure a well-defined discretization of $Z$, $\mathcal{X}$ is validated as follows: A set of valid random points $P_x = \{x_i\}$ in $X$ is encoded to obtain $P_z = \{z_i=h_\text{enc}(x_i)\}$. All boxes in $\cZ$ that contain at least one point in $P_z$ or are immediate neighbors to a box containing a point in $P_z$ are validated. The set $P_x$ can also use states in the demonstrated trajectory data $\cD$.


The \emph{input representation} of the learned dynamics network $\phi^Z_\tau$ is generated by $V(\mathcal{Z})$, the set of all corner points of cubes in $\mathcal{Z}$. The method computes the set of ordered pairs $\Phi_\tau(\mathcal{Z}):=\{(v,\phi_\tau(v)) \mid v \in V(\mathcal{Z})\}$, by calling $h_\text{dyn}(v)$ for time $\tau$ from all $v \in V(\mathcal{Z})$. 
Then, $\Phi_\tau(\mathcal{Z})$ is used to generate the {\it combinatorial representation of the dynamics} of $\phi^Z_\tau$ by approximating it as a {\it combinatorial multivalued map} $\cF \colon \mathcal{Z}\rightrightarrows \mathcal{Z}$, where vertices are $n$-cubes $\xi \in \mathcal{Z}$. The map $\cF$ contains directed edges $\xi \to \xi', \forall\ \xi' \in \Phi_\tau(\xi)$ and the cubes obtained by $\cF(\xi)$ are intended to capture the image of $\phi_\tau^Z(\xi)$. $\cF$ is computed as follows: Given $z\in Z$,  let $\overline{B(z,\delta)}  = \{z'\in Z\mid \|z-z'\| \leq \delta \}$ denote the $\delta$-closed ball at state $z$.
Define the \emph{diameter of $\xi \in \cZ$} by $d(\xi) := \max_{z,z'\in \xi}\|z-z'\|$ and the \emph{diameter of $\cZ$} by $d := \max_{\xi \in \cZ} d (\xi)$. Note that for a uniform grid, $d = d(\xi)$, independently of the choice of $\xi$. Let $V(\xi)$ be the set of corner points of the cube $\xi$ and\\ 

\vspace{-0.125in} 
\noindent $\cF(\xi) :=$
\vspace{-.35in}
\begin{multline}\label{eq:outer_approx}
   \ \ \ \ \ \ \ \left\{ \xi' \mid  \xi' \cap \overline{B\left(\phi^Z_\tau(v),Ld/2 \right)} \neq \emptyset\ \text{for $v\in V(\xi)$}  \right\} \hspace{-.1in} \vspace{-.35in}
\end{multline}

\noindent where $L$ is selected to be an upper bound for the Lipschitz constant $L_\tau$ of $\phi^Z_\tau$. Hence, the above definition of $\cF$ satisfies:

\vspace{-.25in}
\begin{equation}
    \label{eq:OuterApproximation}
    \cF_{min}(\xi) := \{ \xi' \in \mathcal{X}  \mid  \xi'\cap \phi_\tau(\xi) \neq \emptyset \} \subset \cF(\xi)
    \vspace{-.1in}
\end{equation}

\noindent an {\it outer approximation} of $\phi_\tau^Z$. The versatility in defining an outer approximation provides flexibility in incorporating safety constraints by adjusting the parameter $L$, which bounds the Lipschitz constant $L_\tau$.

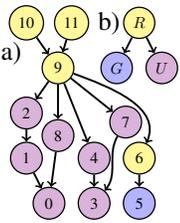
\begin{wrapfigure}{l}{0.27\columnwidth}
    \centering
    \vspace{-0.85in}
    \begin{tikzpicture}
    \begin{scope}
        \node (a) at (-0.2,2) {a)};
        \node (b) at (1.1,2.4) {b)};
    \end{scope}
    \scalebox{0.6}{
    \begin{scope}[every node/.style={circle,thick,draw, minimum size=20pt}]
        \node[fill=violet!30] (0) at (0.5,0) {0};
        \node[fill=violet!30] (1) at (0,1) {1};
        \node[fill=violet!30] (2) at (0,2) {2};
        \node[fill=violet!30] (3) at (1.5,0) {3};
        \node[fill=violet!30] (4) at (1.5,1) {4};
        \node[fill=blue!30] (5) at (2.5,0) {5};
        \node[fill=yellow!50] (6) at (2.5,1) {6};
        \node[fill=violet!30] (7) at (2.2,1.8) {7};
        \node[fill=violet!30] (8) at (0.7,1.5) {8};
        \node[fill=yellow!50] (9) at (0.7,3) {9};
        \node[fill=yellow!50] (10) at (0,4) {10};
        \node[fill=yellow!50] (11) at (1,4) {11};
        \path [->, draw=black, very thick] (11) edge (9);
        \path [->, draw=black, very thick] (10) edge (9);
        \path [->, draw=black, very thick] (9) edge (2);
        \path [->, draw=black, very thick] (9) edge (8);
        \path [->, draw=black, very thick] (9) edge (4);
        \path [->, draw=black, very thick] (9) edge (7);
        \path [->, draw=black, very thick] (9) edge[out=340, in=65] (6);
        \path [->, draw=black, very thick] (2) edge (1);
        \path [->, draw=black, very thick] (8) edge (0);
        \path [->, draw=black, very thick] (7) edge[out=240, in=50] (3);
        \path [->, draw=black, very thick] (1) edge (0);
        \path [->, draw=black, very thick] (4) edge (3);
        \path [->, draw=black, very thick] (6) edge (5);
    \end{scope}
    }
    \scalebox{0.6}{
       \begin{scope}[every node/.style={circle,thick,draw}]
            \node[fill=violet!30] (B) at (3,3) {$U$};
            \node[fill=yellow!50] (C) at (2.5,4) {$R$};
            \node[fill=blue!30] (D) at (2,3) {$G$};
            \path [->, draw=black, very thick] (C) edge (B);
            \path [->, draw=black, very thick] (C) edge (D);
        \end{scope}
        }
    \end{tikzpicture}
    \vspace{-.18in}
    \caption{\small a) A Morse graph $\sMG(\cF)$ for the humanoid of Fig. \ref{fig:overview}; b) bistable Morse graph $\sR_{\sMG}$ from Th.~\ref{thm:order_retraction}.}
    \vspace{-.2in}
    \label{fig:order_retraction}
\end{wrapfigure}
In general, no unique Morse graph can be assigned to a dynamical system; instead, the Morse graph provides a rigorous language for describing the structure of the dynamics at different levels of resolution.
The proposed method aims to identify whether a given initial condition will lead to success or failure for a provided controller.  This can be codified via a Morse graph $\sR_{\sMG}=(\{G, U, R\}, <_r)$ with the partially ordered set (poset) structure $G<_r R$, $U<_r R$, where $G$ denotes success, $U$ denotes failure, and $R$ denotes initial conditions which may lead to success or failure.   This Morse graph has two minimal states, $G$ and $U$.

The number of nodes in the Morse graph $\sMG(\cF)$ computed using Alg. 1 from \cite{morsegraph} with inputs $(\cZ, \Phi_\tau(\cZ), L)$ is typically more than 3 nodes. In particular, it has more than 2 leaf (minimal) nodes, thus providing a richer description of dynamics than required.
Nevertheless, as codified in the following theorem, the information from the poset $(\sMG(\cF),<)$ can be used to produce the desired Morse graph $\sR_{\sMG}$.

\begin{thm}\label{thm:order_retraction}
Define $G$ as the set of minimal nodes of the Morse graph $\sMG(\cF)$ that correspond to successful final states of the system.
Define $R:=\{b\in\sMG \mid g<b \text{ for some $g\in G$}\}$ and set $U = \sMG(\cF)\setminus (G\cup R)$. Then $\sR_{\sMG}=(\{G, U, R\}, <_r)$ with the poset structure $G<_r R$, $U<_r R$ is a Morse graph for the system.
\end{thm}
\begin{proof}
By \cite{order_retraction} it is sufficient to observe (we leave this to the reader) that the map $\rho\colon (\sMG(\cF),<) \to (\{G, U, R\}, <_r)$ given by identification of elements of $\sMG(\cF)$ with elements of $\sR_{\sMG}$ is an order-preserving map.
\end{proof}

Fig. \ref{fig:order_retraction} depicts the Morse graph $\sR_{\sMG} = ({G, U, R}, <r)$ from Theorem \ref{thm:order_retraction}, which is obtained from a Morse graph of the humanoid system exemplified by Fig. \ref{fig:overview}, where $G=\{5\}$, $U=\{0,1,2,3,4,7,8\}$ and $R=\{6,9,10,11\}$. $\sR_{\sMG}$ is also shown in Fig. \ref{fig:overview}.

\noindent {\bf 4. Obtaining the Regions of Attraction} Theorem 1 in \cite{morsegraph} guarantees the existence and uniqueness of \roa$(A)$ for attractors $A$ of $\sMG(\cF)$ and Alg. 2 in \cite{morsegraph} produces \roa$(A)$. The following theorem ensures that the \roa s found by Algorithm 2 can be used to obtain \roa s associated to the Morse graph $\sR_{\sMG}$ from Theorem \ref{thm:order_retraction}. The proof is a consequence of Alg. 2\cite{morsegraph} and Theorem \ref{thm:order_retraction} here.

\begin{thm}
    Let $\cF$ be an outer approximation  of the learned dynamics $\phi^Z_\tau$, $\sMG(\cF)$ be the Morse graph for $\cF$, and $\sR_{\sMG}=(\{G, R, U\}, <_r)$ be the Morse graph from Theorem \ref{thm:order_retraction}. Then $\roa(A)$ and $\roa(U)$ computed via $\sMG(\cF)$ are also regions of attraction of $G$ and $U$ in $\sR_{\sMG}$, respectively.
    
\end{thm}

After applying the proposed method to the pedagogical example described in Section \ref{sec:prelims} and Fig. \ref{fig:bistable}, the bistability is obtained in the latent space as shown in Fig. \ref{fig:bistable}(right). Notice that the level of discretization required to construct the Morse graph for the original space $X$ is quite extensive, amounting to $5\times 3^N$, approximately 2 million cubes for $N=12$, which in general is not well suitable for complex, high-dim. systems. When $\sR_{\sMG}$ is computed on the latent space, however, the size of the required discretization is exponentially reduced. For instance, in Fig.\ref{fig:bistable}(right) the discretization level is $2K$ cubes. This is a reduction of 3 orders of magnitude compared to the smallest discretization needed on $X$ to find the bistability of Fig. \ref{fig:bistable}(left).

\section{Experimental Evaluation}
\label{sec:experiments}

\noindent{\bf Systems:} Fig.~\ref{fig:systems} shows 3 of the considered systems in the experiments in addition to the humanoid of Fig. \ref{fig:overview}. In particular, the experiments consider: (1) A model of a \textbf{Pendulum} (Pend) observed via the coordinates and velocity of its mass $[x,y,\dot{x},\dot{y}]$. (2) A \textbf{Cartpole} (CaPo) simulated using MuJoCo \cite{tunyasuvunakool2020} with state space $[x, \dot{x}, \cos \theta, \sin \theta, \dot{\theta}]$. (3)  The \textbf{Humanoid} (GetUp) benchmark borrowed from the literature \cite{getupcontrol} corresponds to a bipedal humanoid robot attempting a stable standup gait. (4) \textbf{TriFinger Robot Hand} (TriFi) is a real-world dataset of 3 fingers pushing a cube towards a desired location. \cite{guertler2023benchmarking}.

\begin{figure}[h!]
\vspace{-.15in}
    \centering
    \begin{subfigure}{.325\columnwidth}
    \includegraphics[width=\textwidth]{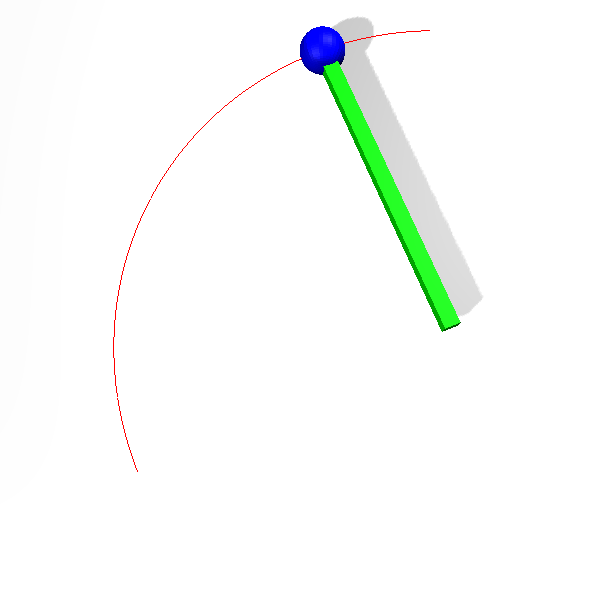}
    \end{subfigure}
    \begin{subfigure}{.325\columnwidth}
    \includegraphics[width=\textwidth]{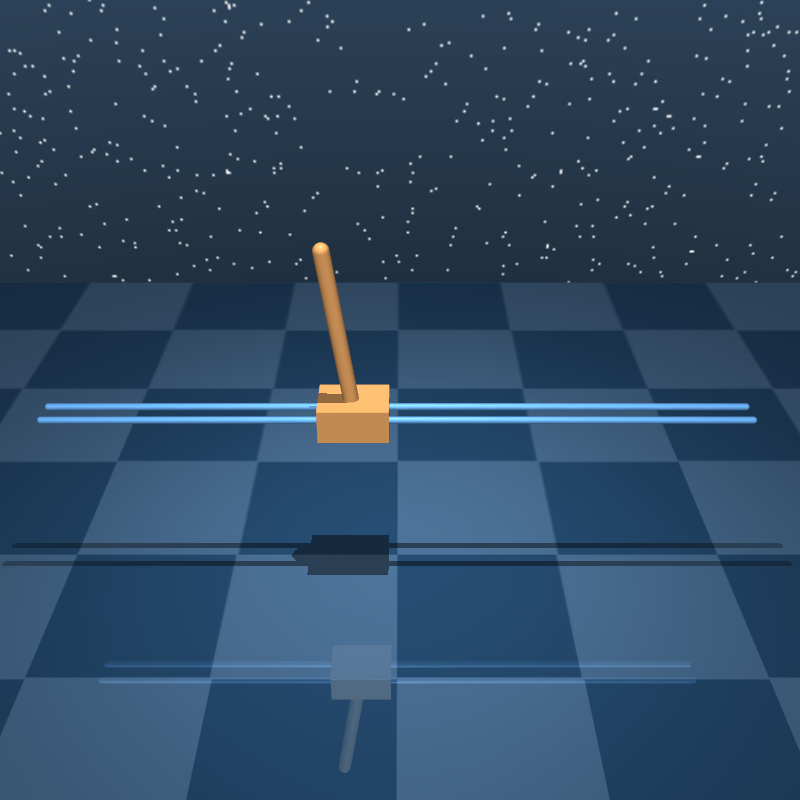}
    \end{subfigure}
    \begin{subfigure}{.325\columnwidth}
    \includegraphics[width=\textwidth,height=\textwidth]{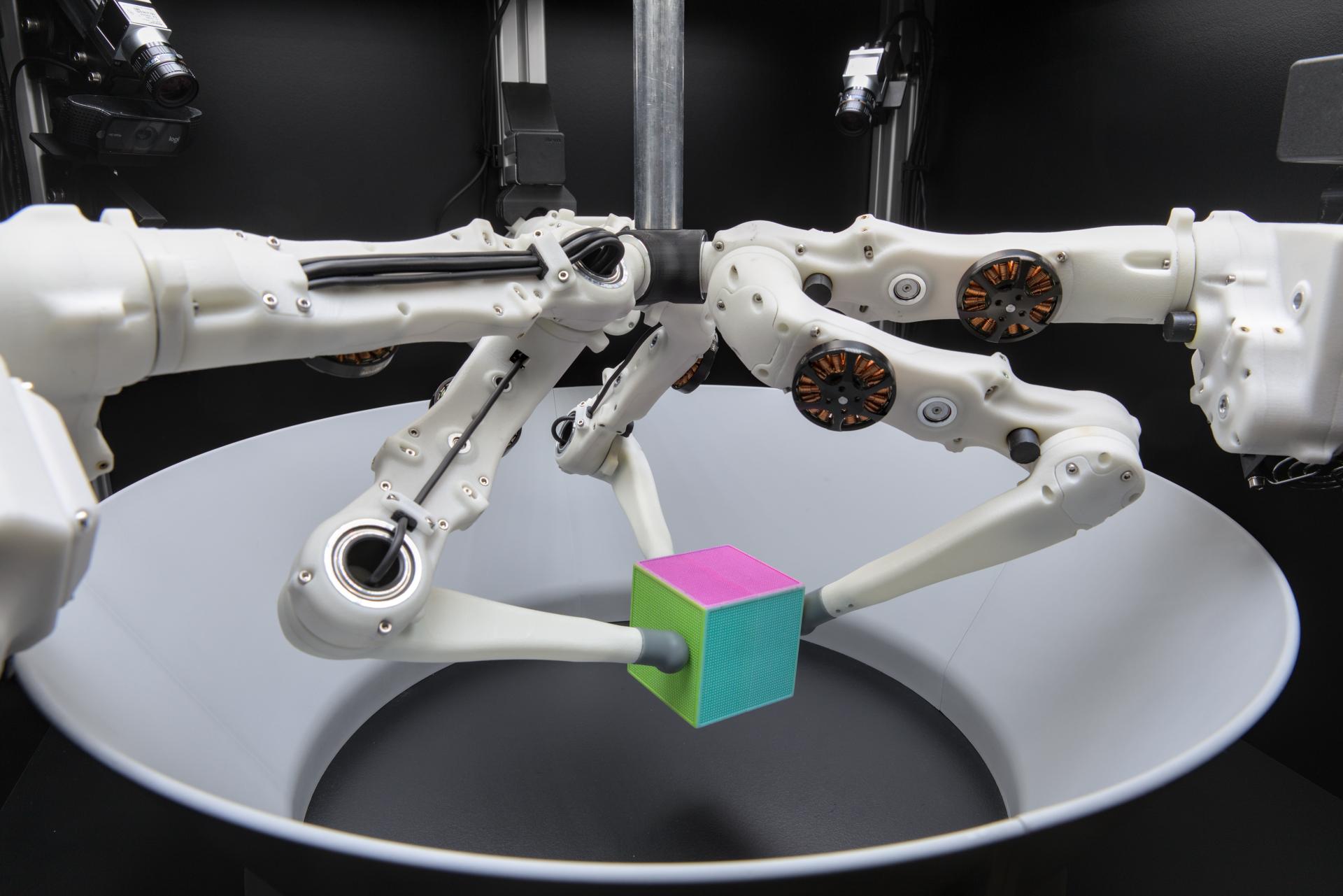}
    \end{subfigure} 
    \caption{\small (L-R) Analytical  Pendulum, Cartpole simulated using MuJoCo, real robot dataset collected using a TriFinger \cite{guertler2023benchmarking}.}
\vspace{-.2in}
    \label{fig:systems}
\end{figure}

\noindent{\bf Controllers:} Both analytical and learned controllers are considered. {\tt LQR} linearizes the system to compute a gain $k$ used in the control law $u(x_t) = -kx_t$. It is applied on the Pend and CaPo systems. Soft Actor-Critic ({\tt SAC}) \cite{Haarnoja2018SoftAO} and Proximal Policy Optimization ({\tt PPO}) \cite{schulman2017proximal} are deep reinforcement learning algorithms trained to maximize the expected return $\mathbb{E}_{x_0}[\sum_{t=0}^{t_\text{max}} \gamma^t \mathcal{R}(x_t, u_t)]$, where $\mathcal{R}: X \times U \mapsto \mathbb{R}$ is a reward function that encodes the goal of the desired task and $\gamma$ is a discount factor. They are applied on the GetUp and TriFi systems, respectively. 



\noindent \textbf{Setup:} For the simulated systems, the datasets $\mathcal{D}$ are obtained by rolling out trajectories using the controller. For TriFi, the dataset is fixed. Given labeled training and test datasets ($\mathcal{D}_{tr},\mathcal{D}_{te}$ - random split 4:1 of trajectories), the autoencoder is trained on $\mathcal{D}_{tr}$ with multiple random seeds for the same hyperparameters and dim($Z$) = 2. $h_\text{enc}, h_\text{dec}, h_\text{dyn}$ are fully-connected multi-layered perceptions (MLPs) with either 2 or 3 hidden layers depending on the benchmark. All trials where \name \ discovered less than 2 attractors were restarted. From the set of desirable final conditions in the training set $\mathbf{F}_f^{tr}$ (as in Section~\ref{sec:network-arch}), all \roa s containing all the points $\{h_\text{enc}(x) \vert x \in \mathbf{F}_f^{tr}\}$ are identified as the desired \roa s. Define $\mathbf{I}^{te}_s$ as the set of initial conditions in $\mathcal{D}_{te}$ that succeed in the task and $\hat{\mathbf{I}}^{te}_s$ as the set of initial conditions in $\mathcal{D}_{te}$ identified by \name\ to be inside the desired \roa. The sets of unsuccessful initial conditions $\mathbf{I}^{te}_f, \hat{\mathbf{I}}^{te}_f$ are similarly defined. Table~\ref{tab:quant} reports the precision ${\tt P} = \frac{\vert \mathbf{I}_s^{te} \cap \hat{\mathbf{I}}_s^{te} \vert}{\vert \hat{\mathbf{I}}_s^{te} \vert}$, recall ${\tt R} = \frac{\vert \mathbf{I}_s^{te} \cap \hat{\mathbf{I}}_s^{te} \vert}{\vert \mathbf{I}_s^{te} \vert}$, and F-score ${\tt F} = \frac{2{\tt P}{\tt R}}{{\tt P + R}}$. It also reports the number of trajectories $\vert \pi_{tr} \vert$ used for training the autoencoder and the number $\vert \mathcal{D}_{tr} \vert$ of $(x,\text{Im}(x))$ pairs obtained by applying a sliding window to the trajectory. All metrics are reported on the testing set $\mathcal{D}_{te}$ after selecting the best-performing hyperparameters on $\mathcal{D}_{tr}$.

\begin{table}[ht]
    \vspace{-.1in}
    \centering
    \begin{tabular}{|c|c|c|c|c|c|c|}
    \hline
     \textbf{Benchmark} & \textbf{dim}($X$) & $\vert \pi_{tr} \vert$ & $\vert \mathcal{D}_{tr} \vert$ &   {\tt P} & {\tt R} & {\tt F}  \\ \hline 
     Pend (LQR) & 4 & 1024 & 20,480 & 94\% & 85\%  & 89\%   \\ \hline \hline
CaPo (LQR) & 5 & 1440 & 143,281  & 86\% & 77\% & 81\% \\ \hline
GetUp (SAC) & 67& 1000 & 326,384 & 91\% & 91\%  & 91\%   \\ \hline
TriFi (PPO) & 96 & 3072 & 460,806 & 90\% & 97\%  & 93\%    \\ \hline
    \end{tabular}
    \vspace{-.05in}
    \caption{\small \name\ performance on benchmarks for dim($Z$) = 2.}
    \vspace{-.15in}
\label{tab:quant}
\end{table}

For  Pend (LQR) and CaPo (LQR), $15-35\%$ of initial conditions in $\mathbf{I}^{tr}_s$ succeed in the task. The learned controllers for GetUp (SAC) and TriFi (PPO) are more successful ($80-90\%$ success rate).  Across the benchmarks, \name\ returns rather accurate estimates of the desired \roa.  



\noindent \textbf{Data Efficiency:} Table~\ref{tab:data-efficiency} varies the number of trajectories used to train the autoencoding network for the Pend (LQR) benchmark (as a ratio of the available data). Increasing the data used by \name \ during the training phase results in an overall improvement in the desired \roa \ estimate. 

\begin{table}[h!]
    \vspace{-.1in}
\begin{minipage}{0.48\columnwidth}
    \centering
    \begin{tabular}{|c|c|c|c|}
    \hline
        \textbf{Size} & {\tt P} & {\tt R} & {\tt F}  \\ \hline
         10\% & 71\% & 58\% & 64\% \\ \hline 
         50\% & 95\% & 79\% & 86\% \\ \hline 
         100\% & 94\% & 85\% & 89\% \\ \hline 
    \end{tabular}
    \vspace{-.05in}
    \caption{\small Data-efficiency for \roa\ estimate of Pend (LQR).}
    \label{tab:data-efficiency}
\end{minipage}%
\hspace{.025in}
\begin{minipage}{0.48\columnwidth}
    \centering
    \begin{tabular}{|c|c|c|c|}
    \hline
        \textbf{Multiplier} & {\tt P} & {\tt R} & {\tt F}  \\ \hline
        $1\times$ & 91\% & 91\% & 91\% \\ \hline 
        $2\times$ & 92\% & 56\% & 70\% \\ \hline 
        $4\times$ & 92\% & 48\% & 63\% \\ \hline 
    \end{tabular}
    \vspace{-.05in}
    \caption{\small Ablation on parameter $L$ for GetUp (SAC).}
    \label{tab:lipschitz}
\end{minipage}%
    \vspace{-.15in}
\end{table}

\noindent \textbf{Selecting a suitable $L$-value:} Table~\ref{tab:lipschitz} varies the parameter $L$ that bounds the Lipschitz constant for GetUp (SAC). Increasing $L$ makes the approach more conservative in accepting positives. This improves precision by reducing the occurrence of False Positives near the \roa's boundary. It comes at the cost of reduced recall, however. 


\begin{table}[h]
  \centering
  \begin{tabular}{|l|ccc|ccc|}
    \hline
    \multicolumn{1}{|c|}{} & \multicolumn{3}{c|}{\textbf{Unlabeled}} & \multicolumn{3}{c|}{\textbf{Labeled}} \\
    \hline
    \textbf{Benchmark} & \texttt{P} & \texttt{R} & \texttt{F} & \texttt{P} & \texttt{R} & \texttt{F} \\
    \hline
    Pend (LQR) & 94\% & 85\% & 89\% & 92\% & 59\% & 72\%  \\ \hline
CaPo (LQR) & 85\% & 76\% & 80\% & 86\% & 77\% & 81\%  \\ \hline
GetUp (SAC) & 91\% & 91\% & 91\% & 82\% & 76\% & 79\%  \\ \hline
TriFi (PPO) & 90\% & 97\% & 93\% & 90\%  & 94\% & 92\% \\ \hline
  \end{tabular}
  \vspace{-.05in}
  \caption{\small Impact of loss term $L_4$, which requires labeling.}
\label{tab:l4-loss}
\vspace{-.15in}
\end{table}

\noindent \textbf{Effect of supervised loss objective: } Table~\ref{tab:l4-loss} reports the accuracy of the \roa\ estimate when the loss function using labeled data as defined in Eqn.~\ref{eqn: l4} is also used. The effect is minimal indicating that \name\ does not need any labels for which demonstration trajectories succeeded or not.


\noindent \textbf{Dimensionality of Latent Space:} Tables~\ref{tab:latent-dim-pendulum}, \ref{tab:latent-dim-humanoid} study the impact of $dim(Z)$ for Pend (LQR) and GetUp (SAC). \name \  recovers a significant portion of the \roa \ for Pend (LQR) even for dim($Z$) = 1 but the estimate improves for dim($Z$) = 2 as the true dynamics lie in a 2-dim. space. For GetUp (SAC), the {\tt R} and {\tt F} scores improve for dim($Z$)=3. But this is due to the entire valid $Z$ returned as the \roa \, leading to more false positives and lower precision.

\begin{table}[h!]
    \vspace{-.1in}
\begin{minipage}{0.48\columnwidth}
    \centering
    \begin{tabular}{|c|c|c|c|}
    \hline
        \textbf{dim}($Z$) & {\tt P} & {\tt R} & {\tt F}  \\ \hline
         1 & 82\% & 86\% & 84\% \\ \hline 
         2 &  94\% & 85\% & 89\% \\ \hline 
    \end{tabular}
    \vspace{-.05in}
    \caption{\small Impact of dim($Z$) for Pend (LQR) \roa.}
    \label{tab:latent-dim-pendulum}
\end{minipage}%
\hspace{.05in}
\begin{minipage}{0.48\columnwidth}
    \centering
    \begin{tabular}{|c|c|c|c|}
    \hline
       \textbf{dim}($Z$) & {\tt P} & {\tt R} & {\tt F}  \\ \hline
         2 & 91\% & 91\% & 91\% \\ \hline 
         3 & 89\%  & 100\%  & 94\%  \\ \hline 
    \end{tabular}
    \vspace{-.05in}
    \caption{\small Impact of dim($Z$) for GetUp (SAC) \roa.}
    \label{tab:latent-dim-humanoid}
\end{minipage}
    \vspace{-.15in}
\end{table}

\noindent \textbf{Weaker controllers:} Table~\ref{tab:weak} reports accuracy for more complex versions of GetUp (SAC) and TriFi (PPO), where the controller works in $50-65\%$ of $X$. For GetUp, the humanoid's max velocity is limited. So some trajectories are erroneously misclassified as failures as they do not stabilize within the specified horizon, impacting the accuracy of \name. The considered version of TriFi requires lifting and transporting the cube to a target instead of the earlier task of pushing. The cube's pose is tracked by noisy cameras, and some critical features may be missing from the dataset resulting in a less accurate $h_\text{dyn}$. On these benchmarks, \name\ achieves high recall at the cost of misclassifying multiple regions of the space as false positives. 

\begin{table}[ht]
    \vspace{-.15in}
    \centering
    \begin{tabular}{|c|c|c|c|c|c|c|}
    \hline
     \textbf{Benchmark} & \textbf{dim}($X$) & $\vert \pi_{tr} \vert$ & $\vert \mathcal{D}_{tr} \vert$ &   {\tt P} & {\tt R} & {\tt F}  \\ \hline 
GetUp (SAC) & 67 & 1000 & 326,384 & 78\% & 96\%  & 86\%   \\ \hline
TriFi (PPO) & 54 & 1915 & 574,831 & 65\%   & 100\%   & 78\%   \\ \hline
    \end{tabular}
    \vspace{-.05in}
    \caption{\small Quantitative evaluation for more complex versions of GetUp (SAC) and TriFi (PPO) (dim($Z$) = 2).}
\label{tab:weak}
    \vspace{-.2in}
\end{table}

\section{Discussion}
\label{sec:conclusion}


\name\ assumes the dataset $\mathcal{D}_\text{tr}$ covers enough initial conditions close to the boundary of success/failure \roa s to approximate the true dynamics. Otherwise, the approach cannot discover the separatrix of the bistable dynamics and discovers only a single attractor. Furthermore, \name\ does not use the decoder $h_\text{dec}$ to obtain the \roa. Future work will explore using $h_\text{dec}$ in the high-dim. space with low-accuracy predictions to mitigate False Positives. Finally, \name\ does not provide a conservative estimate of the ground truth \roa\ due to the stochasticity in training the autoencoder. It is interesting to explore further how to maximize precision.  

The \roa s discovered by \name\ for different controllers of the same robotics system can be used to compose a hybrid solution that succeeds in the task from a wider swath of the state space. Given a fixed dataset of robotics transitions, \name\ can discover the regions of the robot's dynamics from where task success is feasible, which has applications in both safe controller learning \cite{gu2023review} and safe motion planning \cite{Knuth-ICRA-23}. Future directions include learning safety regions for constrained systems \cite{dawson2023safe} and deployment in real robotic systems, where challenges related to system dynamics approximation and data efficiency must be addressed.



\bibliographystyle{IEEEtran}
\bibliography{refs.bib}

\end{document}